\documentclass[11pt]{article}

\usepackage[utf8]{inputenc}
\usepackage[T1]{fontenc}
\usepackage{lmodern}

\usepackage{microtype}
\usepackage{graphicx}
\usepackage{booktabs}
\usepackage{multirow}
\usepackage{subcaption}
\usepackage{caption}

\usepackage{amsmath}
\usepackage{amssymb}
\usepackage{amsthm}
\usepackage{bm}
\usepackage{mathtools}

\usepackage{algorithm}
\usepackage{algpseudocode}

\usepackage{enumitem}
\usepackage{url}
\usepackage[hidelinks]{hyperref}

\usepackage[numbers]{natbib}

\usepackage[margin=1in]{geometry}

\usepackage[table]{xcolor}

\definecolor{goldrank}{RGB}{255,215,0}
\definecolor{silverrank}{RGB}{192,192,192}
\definecolor{bronzerank}{RGB}{205,127,50}

\DeclareMathOperator*{\argmax}{arg\,max}

\newtheorem{theorem}{Theorem}
\newtheorem{proposition}{Proposition}
\newtheorem{definition}{Definition}

\title{
Complement Submodular Information Measures for Balanced and Robust Data Selection
}

\author{
Rishabh Iyer \\
Department of Computer Science \\
The University of Texas at Dallas \\
Richardson, TX 75080 \\
\texttt{rishabh.iyer@utdallas.edu}
}
\date{}

\begin{document}

\maketitle

\begin{abstract}

Submodular optimization has become a fundamental paradigm for data selection, retrieval, summarization, and representation learning due to its ability to model coverage, diversity, and representativeness. However, classical submodular objectives optimize only the selected subset and do not explicitly preserve structural information between the selected subset and the remaining data. In many modern machine learning applications, including train/validation/test splitting, benchmark construction, and robust subset selection, the quality of a selection depends critically on preserving balanced structure across both the selected subset and its complement. 

In this work, we introduce \emph{Complement Submodular Information} (CSI), a new class of complement-aware submodular objectives that quantify shared structural information between a subset and its complement. Our framework induces complement-aware variants of several classical submodular functions including Facility Location, Graph Cut, LogDet, Saturated Coverage, Set Cover, Probabilistic Set Cover, and Feature Based Functions. We analyze the theoretical properties of CSI objectives and show that they exhibit approximate monotonicity under bounded curvature conditions, leading to near-$(1-1/e)$ greedy approximation guarantees.

Empirically, CSI objectives consistently outperform standard submodular objectives on robust hidden-slice-aware subset selection. In particular, CSI objectives significantly improve preservation of coherent rare/tail semantic structure while simultaneously suppressing noisy and isolated outliers, leading to substantially improved downstream predictive performance. Synthetic experiments further illustrate how different CSI instantiations capture complementary notions of representativeness, diversity, connectivity, and balanced neighborhood preservation.
\end{abstract}

\section{Introduction}

Submodular optimization has emerged as a fundamental paradigm in machine learning and artificial intelligence due to its natural ability to model coverage, diversity, representativeness, and information \cite{krause2008near, fujishige2005submodular, bilmes2022submodularity}. Over the past decade, submodular objectives and submodular information measures (SMI)~\cite{iyer2021submodular, iyer2021generalized} have been widely used in applications including active learning, dataset pruning, retrieval, representation learning, distribution shift adaptation, and benchmark construction \cite{kothawade2021similar, kothawade2022talisman, karanam2022orient, smola2026submodular}. Classical submodular objectives such as Facility Location, Graph Cut, LogDeterminant, and Set Cover have proven particularly effective for selecting representative or diverse subsets under limited labeling, compute, or memory budgets.

Despite their success, existing submodular objectives fundamentally optimize only the selected subset itself. However, in many modern machine learning settings, the quality of a selection depends not only on the selected subset, but also on the structure preserved in the remaining data. For example, in train/validation/test splitting and robust subset selection, one seeks balanced preservation of representative structure across both head and tail semantic slices while simultaneously suppressing noisy or isolated outliers. Standard representative objectives may \emph{over-focus on dominant regions, while diversity-oriented objectives can over-select outliers that appear highly dissimilar from the rest of the dataset}.

\begin{figure}[t]
    \centering
    \includegraphics[width=0.88\textwidth]{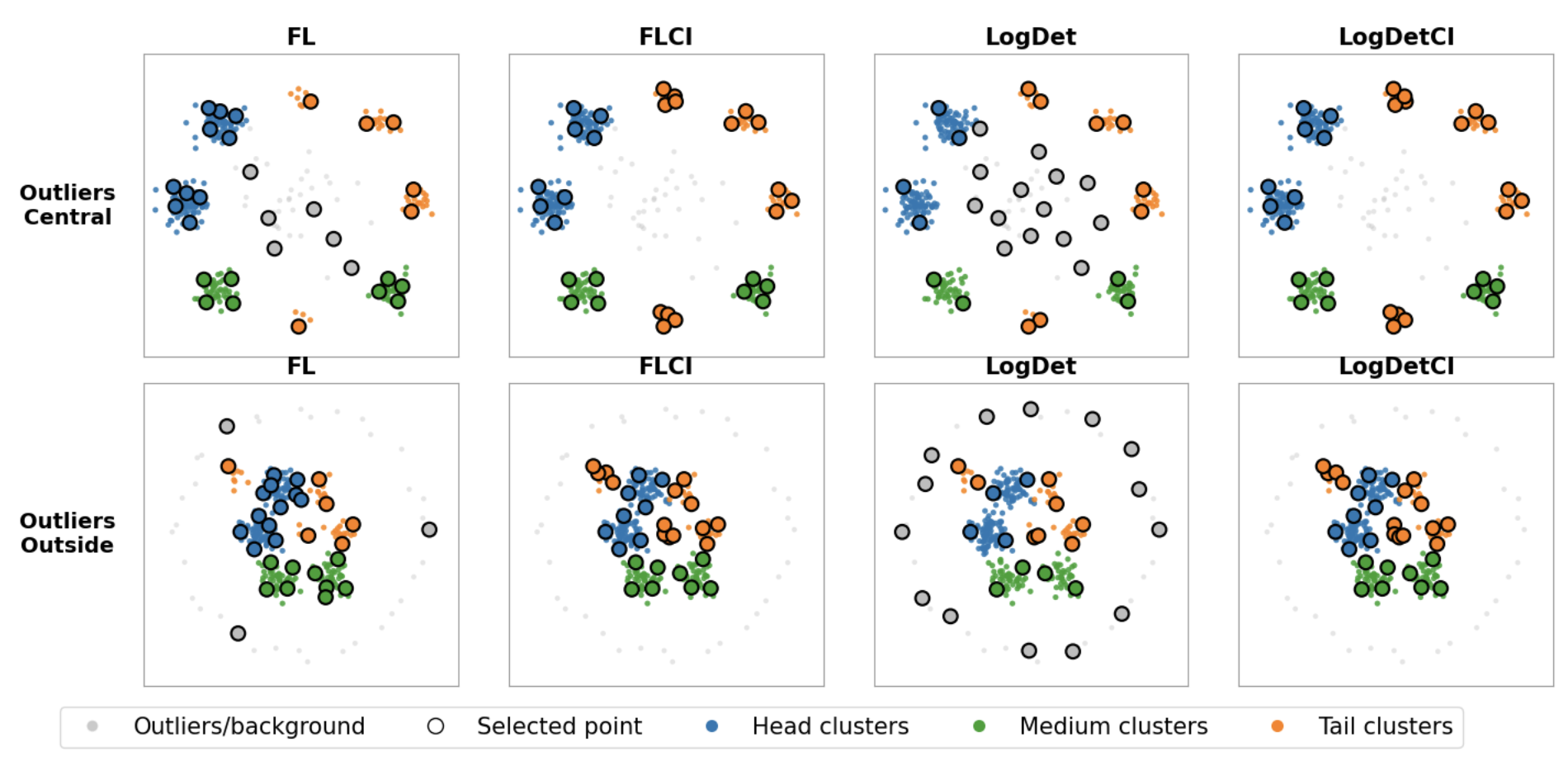}
    \caption{
    Synthetic illustration comparing standard submodular objectives and complement-aware CSI objectives. Standard diversity objectives such as LogDet tend to over-select isolated outliers, while representative objectives such as Facility Location (FL) may bias toward dominant regions or central background structure. In contrast, CSI objectives preserve information jointly across the selected subset and its complement, leading to balanced preservation of head and tail semantic structure while suppressing outliers.
    }
    \label{fig:csi_main}
\end{figure}

To address this limitation, we introduce a new class of objectives called \emph{Complement Submodular Information} (CSI) measures. Given a normalized monotone submodular function $f$, we define the complement information between a subset $A$ and its complement $V \setminus A$ as:
\[
I_f(A;V \setminus A)
=
f(A) + f(V \setminus A) - f(V).
\]
CSI objectives can be interpreted as complement-aware analogues of submodular mutual information measures, encouraging structural preservation jointly across the selected subset and the remaining data. The framework naturally induces complement-aware variants of classical submodular functions including Facility Location, Graph Cut, LogDet, Probabilistic Set Cover, Set Cover, Saturated Coverage, and Feature Based Functions~\cite{kulesza2012determinantal,cornuejols1983uncapicitated,feige1998threshold,wei2014unsupervised,iyer2015submodular}.

Objectives of the form $f(A)+f(V\setminus A)$ have appeared implicitly or explicitly in prior works on graph partitioning, determinant-based diversity modeling, Gaussian process sensor placement, and more recently benchmark selection \cite{krause2008near, queyranne1998minimizing, smola2026submodular}. However, these works focus on specialized instances rather than a unified complement-aware information framework. In contrast, our work provides a systematic treatment of CSI objectives including their theoretical properties, optimization behavior, function-specific interpretations, and applications to balanced splitting and robust selection.

From a theoretical perspective, CSI objectives are symmetric submodular functions that are generally non-monotone. Nevertheless, we show that CSI objectives exhibit approximate monotonicity under bounded curvature conditions, leading to curvature-dependent near-$(1-1/e)$ approximation guarantees for greedy optimization. Empirically, CSI objectives consistently outperform their corresponding underlying submodular objectives on synthetic experiments, hidden-slice-aware train/validation/test splitting, and robust subset selection. In particular, CSI objectives substantially improve balance across rare/tail and head semantic slices while simultaneously suppressing isolated outliers. Importantly, these semantic slices are not explicitly labeled and arise purely through latent structure in the embedding space.

The main contributions of this work are summarized below:
\begin{itemize}
    \item We introduce Complement Submodular Information (CSI), a unified framework for complement-aware submodular objectives.

    \item We derive complement-aware variants of Facility Location, Graph Cut, LogDet, and related objectives together with curvature-based near-$(1-1/e)$ approximation guarantees.

    \item We demonstrate that CSI objectives significantly improve robust hidden-slice-aware subset selection compared to classical submodular objectives and random selection.
\end{itemize}

\section{Preliminaries}

Let $V=\{1,2,\dots,n\}$ denote a finite ground set. A set function $f:2^V \rightarrow \mathbb{R}$ is called \emph{submodular} if for all $A \subseteq B \subseteq V$ and $e \in V \setminus B$,
\[
f(e \mid A) \ge f(e \mid B),
\]
where
\[
f(e \mid A) \triangleq f(A \cup \{e\}) - f(A)
\]
denotes the marginal gain of adding element $e$ to set $A$. Intuitively, submodularity models the diminishing returns property. A function is \emph{monotone} if $f(e \mid A) \ge 0$ for all $A \subseteq V$ and $e \notin A$, and is \emph{normalized} if $f(\emptyset)=0$.

Submodular functions have been widely used for modeling representativeness, diversity, and coverage in machine learning \cite{fujishige2005submodular, krause2014submodular, bilmes2022submodularity,iyer2015polyhedral}. Common examples include Facility Location, Graph Cut, LogDet, and Set Cover functions. Under a cardinality constraint $|A|\le k$, monotone submodular maximization admits a $(1-1/e)$ approximation guarantee using the classical greedy algorithm \cite{nemhauser1978analysis}.

Submodular information measures (SMI) \cite{iyer2021submodular, iyer2021generalized} generalize classical notions of information using submodular functions. Given two sets $A,Q \subseteq V$, the submodular mutual information induced by a submodular function $f$ is defined as
$I_f(A;Q)
=
f(A) + f(Q) - f(A \cup Q)$.
Similarly, the conditional submodular gain or conditional entropy-style quantity is defined as
$H_f(A \mid P)
=
f(A \cup P) - f(P)$,
where $P \subseteq V$ is a conditioning set. Intuitively, $H_f(A \mid P)$ measures the additional information or utility provided by $A$ beyond what is already captured by $P$.

Several parameterized and function-specific SMI instantiations have been proposed, including Facility Location Mutual Information (FLMI), Facility Location Variant Mutual Information (FLQMI), Graph Cut Mutual Information (GCMI), and Concave Over Modular (COM) objectives. These formulations have been successfully used in targeted active learning, distribution shift adaptation, retrieval-aware learning, and robust subset selection \cite{kothawade2021similar, kothawade2022talisman, karanam2022orient, kothawade2022active}.

In this work, we study a complementary setting where the ``query'' itself is the complement of the selected subset. This leads to a new class of complement-aware objectives that preserve structural information jointly across a subset and its complement.

\section{Complement Submodular Information Measures}

In many machine learning applications, the quality of a subset depends not only on the structure preserved within the selected subset itself, but also on the relationship between the selected subset and the remaining data. Motivated by this observation, we introduce a new class of objectives called \emph{Complement Submodular Information} (CSI) measures.

\subsection{Definition}

Let $f:2^V \rightarrow \mathbb{R}$ be a normalized monotone submodular function~\cite{fujishige2005submodular}. Given a subset $A \subseteq V$, we define the complement submodular information induced by $f$ as:
\[
I_f(A;V\setminus A)
=
f(A)+f(V\setminus A)-f(V).
\]

CSI objectives can be interpreted as complement-aware analogues of submodular mutual information measures, encouraging structural preservation jointly across the selected subset and the remaining data. Unlike classical submodular objectives that optimize only the selected subset itself, CSI objectives explicitly preserve information across both sides of the partition.

\subsection{Properties}

We first establish several important properties of CSI objectives.

\begin{proposition}
Let $f$ be normalized and submodular. Then the CSI objective $g(A)=I_f(A;V\setminus A)$
is symmetric: $g(A)=g(V\setminus A)$.
\end{proposition}

\begin{proposition}
If $f$ is normalized and submodular, then $g(A)=I_f(A;V\setminus A)$
is submodular.
\end{proposition}

While CSI objectives are generally non-monotone, they often exhibit approximate monotonicity in practice. The marginal gain of adding an element $e \notin A$ is: $g(e\mid A)
=
f(e\mid A)
-
f(e\mid V\setminus(A\cup\{e\}))$.

Thus, CSI becomes non-monotone only when an element contributes significantly more strongly to the complement than to the selected subset itself. We show later in the paper that under bounded curvature conditions, CSI objectives exhibit approximate monotonicity and admit near-$(1-1/e)$ approximation guarantees for greedy optimization.

CSI objectives naturally encourage balanced preservation of structure across both partitions. In particular, complement-aware diversity objectives such as LogDetCI suppress isolated outliers while preserving coherent rare semantic modes, while representative objectives such as FLCI encourage balanced coverage across head and tail semantic slices.

\section{Instantiations of CSI Objectives}

The CSI framework naturally induces complement-aware variants of a broad family of submodular objectives. Different CSI instantiations preserve different notions of structure including representativeness, diversity, connectivity, and balanced neighborhood coverage. Table~\ref{tab:csi_functions} summarizes the major CSI objectives studied in this work together with their induced structural behavior.

\begin{table*}[t]
\centering
\caption{Summary of CSI instantiations and their structural properties.}
\label{tab:csi_functions}
\resizebox{\textwidth}{!}{
\begin{tabular}{l l l l}
\toprule
\textbf{Objective} & \textbf{Definition} & \textbf{Primary Property} & \textbf{Behavior} \\
\midrule

FLCI &
$
\sum_{i \in V}
\min\left(
\max_{j \in A} s_{ij},
\max_{j \in V \backslash A} s_{ij}
\right)
$
&
Representative coverage &
Balanced representative partitioning \\

GCCI &
$
\sum_{i \in A, j \in V \backslash A} s_{ij}
$
&
Connectivity preservation &
Balanced graph partitioning \\

LogDetCI &
$
\log\det(S_A)+\log\det(S_{V\setminus A})-\log\det(S_V)
$
&
Diversity preservation &
Rare mode preservation with outlier suppression \\

PSCCI &
$
\sum_{i\in V}
\left(1 - \prod_{j\in A}(1-p_{ij})\right)
\left(1 - \prod_{j\in V \backslash A}(1-p_{ij})\right)
$
&
Probabilistic coverage &
Balanced probabilistic coverage \\

SCCI &
$
\sum_{i\in V}
\min\left\{
\sum_{j\in A}s_{ij},\;
\sum_{j\in V\setminus A}s_{ij},\;
\alpha,\;
\left(
\sum_{j\in V}s_{ij}-\alpha
\right)_+
\right\}
$
&
Saturating neighborhood coverage &
Balanced local neighborhood preservation \\

FBCI &
$
f_{\text{FB}}(A)+f_{\text{FB}}(V\setminus A)-f_{\text{FB}}(V),
\quad
f_{\text{FB}}(A)=
\sum_{\ell=1}^{d}
\psi\left(\sum_{j\in A}x_{j\ell}\right)
$
&
Feature coverage &
Balanced semantic feature preservation \\
\bottomrule
\end{tabular}
}
\end{table*}

\subsection{Facility Location Complement Information (FLCI)}

The classical Facility Location function is defined as:
\[
f_{\text{FL}}(A)
=
\sum_{i\in V}\max_{j\in A}s_{ij},
\]
where $s_{ij}$ denotes pairwise similarity between items $i$ and $j$.

Applying the CSI framework yields:
\begin{align}
\text{FLCI}(A)
&=
f_{\text{FL}}(A)
+
f_{\text{FL}}(V\setminus A)
-
f_{\text{FL}}(V)
\nonumber \\
&=
\sum_{i \in V}
\min\left(
\max_{j \in A} s_{ij},
\max_{j \in V \backslash A} s_{ij}
\right).
\end{align}

FLCI encourages representative structure to be preserved jointly across both partitions. Since a point contributes only when it is well represented by both the selected subset and its complement, FLCI naturally produces balanced representative partitions across head and tail semantic regions. Moreover, isolated outliers contribute little complement information because they typically lack strong representative support from both sides of the partition.

\subsection{Graph Cut Complement Information (GCCI)}

The Graph Cut function is defined as:
\[
f_{\text{GC}}(A)
=
\sum_{i\in V,j\in A}s_{ij}
-
\lambda\sum_{i,j\in A}s_{ij}.
\]

The CSI variant becomes:
\begin{align}
\text{GCCI}(A)
&=
f_{\text{GC}}(A)
+
f_{\text{GC}}(V\setminus A)
-
f_{\text{GC}}(V)
\nonumber \\
&=
\sum_{i \in A, j \in V \backslash A} s_{ij}.
\end{align}

GCCI measures connectivity preserved across the partition boundary and is closely related to balanced graph partitioning objectives. The objective favors subsets that maintain strong cross-partition interactions while discouraging disconnected or isolated structures.

\subsection{LogDet Complement Information (LogDetCI)}

The LogDet function is defined as:
\[
f_{\text{LD}}(A)
=
\log\det(S_A),
\]
where $S_A$ denotes the similarity submatrix corresponding to subset $A$.

The CSI variant becomes:
\[
\text{LogDetCI}(A)
=
\log\det(S_A)
+
\log\det(S_{V\setminus A})
-
\log\det(S_V).
\]

Standard LogDet objectives aggressively favor diversity and therefore may over-select isolated outliers that appear highly dissimilar from the remaining data. In contrast, LogDetCI preserves diversity jointly across both partitions. Coherent rare semantic clusters contribute strongly because they preserve determinant volume on both sides of the partition, while isolated outliers contribute little complement diversity. As a result, LogDetCI naturally balances diversity, rare mode preservation, and outlier suppression.

\subsection{Probabilistic Set Cover Complement Information (PSCCI)}

The Probabilistic Set Cover function is:
\[
f_{\text{PSC}}(A)
=
\sum_{i\in V}
\left(
1-
\prod_{j\in A}(1-p_{ij})
\right),
\]
where $p_{ij}$ denotes probabilistic coverage.

The CSI variant becomes:
\begin{align}
\text{PSCCI}(A)
&=
f_{\text{PSC}}(A)
+
f_{\text{PSC}}(V\setminus A)
-
f_{\text{PSC}}(V)
\nonumber \\
&=
\sum_{i\in V}
\left(
1-\prod_{j\in A}(1-p_{ij})
\right)
\left(
1-\prod_{j\in V \backslash A}(1-p_{ij})
\right).
\end{align}

PSCCI rewards concepts or instances that are probabilistically covered by both the selected subset and its complement. Consequently, the objective favors balanced probabilistic coverage across the partition rather than concentrating coverage entirely within one side. If $p_{ij}$ are binary ($0$ or $1$), this is probabilistic set cover gives the standard set cover function and PSCCI will become the Set Cover CI function. 

\subsection{Saturated Coverage Complement Information (SCCI)}

For similarity-based coverage, the Saturated Coverage function is:
\[
f_{\text{SC}}(A)
=
\sum_{i\in V}
\min\left(
\alpha,\sum_{j\in A}s_{ij}
\right),
\]
where $\alpha$ is a saturation threshold.

Applying the CSI framework yields:
\[
\text{SCCI}(A)
=
f_{\text{SC}}(A)
+
f_{\text{SC}}(V\setminus A)
-
f_{\text{SC}}(V).
\]

Let
\[
a_i=\sum_{j\in A}s_{ij},
\qquad
b_i=\sum_{j\in V\setminus A}s_{ij}.
\]
Then SCCI admits the equivalent form:
\[
\text{SCCI}(A)
=
\sum_{i\in V}
\min\left\{
a_i,\;
b_i,\;
\alpha,\;
(a_i+b_i-\alpha)_+
\right\},
\]
where $(x)_+ = \max(x,0)$.

This expression gives a particularly clean interpretation of complement-aware saturated coverage. A point contributes only when it receives meaningful neighborhood support from both the selected subset and its complement. Dense head-region clusters saturate quickly, reducing the benefit of repeatedly selecting redundant points from dominant regions. In contrast, coherent rare semantic clusters provide mutual neighborhood support across both sides of the partition and therefore contribute strongly to the objective. Isolated outliers typically lack sufficient similarity support from either side and consequently contribute little complement information.

Thus, SCCI naturally encourages balanced local neighborhood preservation while simultaneously suppressing redundant head-region selections and isolated noisy outliers.

\paragraph{Feature-Based Complement Information (FBCI).}
Feature-based submodular functions provide another natural CSI instantiation:
\[
f_{\text{FB}}(A)
=
\sum_{\ell=1}^d
\psi\left(\sum_{j\in A} x_{j\ell}\right),
\]
where $x_{j\ell}\ge 0$ denotes the activation of feature $\ell$ for item $j$ and $\psi$ is concave. The corresponding CSI objective is
\[
\text{FBCI}(A)
=
f_{\text{FB}}(A)
+
f_{\text{FB}}(V\setminus A)
-
f_{\text{FB}}(V).
\]
FBCI encourages semantic feature activations to be preserved across both the selected subset and its complement. The concavity suppresses redundant head features while allowing coherent rare features to contribute, making FBCI useful for balanced feature-space partitioning and hidden-slice preservation.

\section{Optimization}
\label{sec:optimization}

CSI objectives are symmetric submodular functions that are generally non-monotone. Nevertheless, they often exhibit approximate monotonicity in practice, enabling efficient greedy optimization with strong approximation guarantees.

\subsection{Greedy Optimization}

Given a CSI objective
\[
g(A)=I_f(A;V\setminus A),
\]
our goal is to solve:
\[
\max_{|A|\le k} g(A).
\]

We optimize CSI objectives using the standard greedy algorithm~\cite{nemhauser1978analysis}, which iteratively adds the element with the largest marginal gain:
\[
e^\star
=
\argmax_{e\in V\setminus A}
g(e\mid A).
\]

The marginal gain of adding an element $e\notin A$ is:
\[
g(e\mid A)
=
f(e\mid A)
-
f(e\mid V\setminus(A\cup\{e\})).
\]

Unlike classical monotone submodular objectives, CSI objectives can exhibit negative marginal gains when an element contributes more strongly to the complement than to the selected subset itself. However, in many practical settings this non-monotonicity is limited.

\subsection{Approximate Monotonicity}

We formalize this notion using approximate monotonicity~\cite{iyer2015submodular}.

\begin{definition}
A CSI objective $g$ is said to be $\epsilon$-approximately monotone under cardinality constraint $k$ if for all $A\subseteq V$ with $|A|<k$ and all $e\notin A$,
\[
g(e\mid A)\ge -\epsilon.
\]
\end{definition}

The degree of non-monotonicity is controlled by the curvature of the underlying submodular function~\cite{conforti1984submodular,iyer2013curvature}.

\begin{theorem}
Let $f$ be normalized, monotone, and submodular with restricted curvature $\kappa_k(f)$. Then the CSI objective
\[
g(A)=I_f(A;V\setminus A)
\]
is $\epsilon$-approximately monotone with
\[
\epsilon
\le
\kappa_k(f)\max_{e\in V} f(e).
\]
\end{theorem}

Intuitively, low curvature implies that the marginal contribution of an element does not vary significantly across subsets, thereby limiting the degree of complement-induced non-monotonicity.

Approximate monotonicity immediately yields a near-$(1-1/e)$ approximation guarantee for greedy optimization.

\begin{theorem}
Let $A_g$ denote the subset returned by greedy optimization under cardinality constraint $k$, and let
\[
A^\star \in \argmax_{|A|\le k} I_f(A;V\setminus A).
\]
Then
\[
I_f(A_g;V\setminus A_g)
\ge
\left(1-\frac{1}{e}\right)
I_f(A^\star;V\setminus A^\star)
-
\frac{k\epsilon}{e}.
\]
\end{theorem}

The proof of the above theorems are in Appendix A. Combining the above results yields:
\[
I_f(A_g;V\setminus A_g)
\ge
\left(1-\frac{1}{e}\right)
I_f(A^\star;V\setminus A^\star)
-
\frac{k}{e}\kappa_k(f)\max_e f(e).
\]

Thus, when the underlying submodular function has bounded curvature, CSI objectives behave nearly monotone and greedy optimization achieves near-$(1-1/e)$ guarantees. This generalizes the results for specific CSI functions shown in~\cite{krause2008near,lin2010multi,smola2026submodular}

\subsection{Efficient Implementation}

Many CSI objectives admit efficient memoized implementations similar to classical submodular functions. In particular, Facility Location and Graph Cut based objectives support efficient incremental gain updates, while LogDet objectives can be accelerated using rank-one determinant updates.

In practice, we combine memoization with lazy greedy optimization \cite{minoux1978accelerated,iyer2019memoization}, significantly reducing computational complexity while preserving exact greedy solutions. Empirically, the runtime of CSI optimization is comparable to that of the corresponding underlying submodular objectives.

\section{Experiments}

We evaluate CSI objectives on two complementary settings:
(i) controlled synthetic experiments for understanding the structural behavior induced by different CSI objectives,
(ii) robust hidden-slice-aware data selection. 

We compare standard submodular objectives against their CSI counterparts including Facility Location (FL/FLCI), Graph Cut (GC/GCCI), LogDet (LogDet/LogDetCI), Probabilistic Set Cover (PSC/PSCCI), Saturated Coverage (SC/SCCI), and Feature-Based functions (FB/FBCI). For the downstream real-data experiments, we focus primarily on the strongest-performing CSI objectives including FLCI, LogDetCI, and SCCI.

\subsection{Synthetic Experiments}

We first conduct controlled synthetic experiments to understand the structural behavior induced by different CSI objectives. Our primary goal is to study how complement-aware objectives differ from their corresponding underlying submodular functions in terms of balanced selection, rare semantic slice preservation, and outlier suppression.

\paragraph{Synthetic Setup.}
We generate synthetic datasets consisting of multiple Gaussian clusters with varying cluster sizes to simulate head, medium, and tail semantic regions. In addition, we introduce isolated outliers either near the center of the embedding space or surrounding the clusters. Importantly, the semantic slices are not explicitly labeled during selection and arise purely through latent geometric structure. Similarity matrices are constructed using RBF kernels over the synthetic embeddings. For FB function, we use the similarity across the entire dataset as the feature vector ($x_{jl} = s_{jl}$ and $d = |V|$).

Figure~\ref{fig:csi_main} visualizes representative selections obtained using FL, FLCI, LogDet, and LogDetCI. Standard diversity-oriented objectives such as LogDet aggressively select isolated outliers because they maximize geometric diversity, while standard representative objectives such as FL may over-focus on dominant regions or central background structure. In contrast, CSI objectives preserve balanced structure jointly across the selected subset and its complement, leading to improved preservation of coherent rare semantic clusters while suppressing isolated outliers.

\begin{table*}[t]
\centering
\caption{Synthetic experiment results comparing standard submodular objectives and CSI objectives. Higher minority coverage and lower outlier rate indicate improved balanced structure preservation. Best, second-best, and third-best results are highlighted in gold, silver, and bronze respectively.}
\label{tab:synthetic_results}

\resizebox{\textwidth}{!}{
\begin{tabular}{lccccc}
\toprule
\textbf{Method} &
\textbf{Minority Coverage} $\uparrow$ &
\textbf{Outlier Rate} $\downarrow$ &
\textbf{Selected vs Full KL} $\downarrow$ &
\textbf{Selected vs Complement KL} $\downarrow$ &
\textbf{Coverage Distance} $\downarrow$ \\
\midrule

FL 
& 0.95 
& 0.24 
& \cellcolor{bronzerank!35}0.006 
& 0.037 
& 0.77 \\

FLCI 
& 1.45 
& \cellcolor{bronzerank!35} 0.12 
& \cellcolor{goldrank!35}\textbf{0.002} 
& \cellcolor{goldrank!35}\textbf{0.024} 
& \cellcolor{goldrank!35}\textbf{0.47} \\

GC 
& 0.51 
& 0.41 
& 0.047 
& 0.072 
& 1.07 \\

GCCI 
& 0.60 
& 0.34 
& 0.035 
& 0.072 
& 1.15 \\

LogDet 
& 1.89 
& 0.36 
& 0.067 
& 0.076 
& 0.92 \\

LogDetCI 
& \cellcolor{goldrank!35}\textbf{3.79} 
& 0.14 
& 0.046 
& 0.218 
& 0.76 \\

PSC 
& 0.60 
& 0.24 
& 0.012 
& \cellcolor{bronzerank!35}0.030 
& 2.06 \\

PSCCI 
& 0.80 
& 0.14 
& \cellcolor{silverrank!35}0.009 
& \cellcolor{silverrank!35}0.028 
& 1.06 \\

SC 
& 1.12 
& 0.18 
& 0.041 
& 0.052 
& 0.71 \\

SCCI 
& \cellcolor{silverrank!35}2.85 
& \cellcolor{goldrank!35}\textbf{0.06} 
& 0.063 
& 0.079 
& \cellcolor{silverrank!35} 0.54 \\

FB 
& 1.34 
& 0.16 
& 0.053 
& 0.067 
& 0.69 \\

FBCI 
& \cellcolor{bronzerank!35}2.67 
& \cellcolor{silverrank!35}0.08 
& 0.071 
& 0.088 
& \cellcolor{bronzerank!35}0.57 \\

\bottomrule
\end{tabular}
}
\end{table*}
\paragraph{Evaluation Metrics.}
To characterize the structural behavior of the selected subsets, we evaluate the following metrics (Appendix B has the formal mathematical definitions):

\begin{itemize}
    \item \textbf{Minority Coverage Ratio:} Fraction of rare/tail slice instances selected relative to their proportion in the dataset.

    \item \textbf{Outlier Selection Rate:} Fraction of selected points corresponding to isolated synthetic outliers.

    \item \textbf{Selected vs Full KL:} KL divergence between the cluster distribution of the selected subset and the full dataset.

    \item \textbf{Selected vs Complement KL:} KL divergence between the selected subset and its complement.

    \item \textbf{Manifold Coverage Distance:} Average nearest-neighbor distance between the selected subset and the full dataset manifold.
\end{itemize}

\paragraph{Results and Insights.}

Several important trends emerge from Table~\ref{tab:synthetic_results}. First, across nearly all objective families, the CSI variants consistently improve balanced preservation of rare semantic structure compared to their corresponding underlying submodular functions. In particular, FLCI, LogDetCI, SCCI, and FBCI substantially improve minority slice coverage while simultaneously reducing outlier selection and being more representative. This demonstrates that complement-aware optimization naturally regularizes the tendency of classical submodular objectives to either over-focus on dominant regions or aggressively select isolated outliers.

Among all methods, LogDetCI achieves the strongest minority slice preservation, significantly outperforming standard LogDet while also reducing outlier selection by nearly a factor of three. Standard LogDet aggressively favors geometric diversity and therefore tends to select isolated extreme points. In contrast, LogDetCI preserves diversity jointly across the selected subset and its complement, leading to coherent rare-mode preservation rather than pure outlier-driven diversity.

Representative CSI objectives such as FLCI, SCCI, and FBCI achieve the strongest balance between representative coverage and robustness. FLCI significantly improves representative balance relative to FL while achieving the best KL alignment between the selected subset and the full dataset. Similarly, SCCI and FBCI substantially improve rare slice preservation while simultaneously achieving the lowest outlier selection rates and strongest manifold coverage. Since these objectives reward neighborhood or feature support jointly across both partitions, coherent rare semantic regions contribute strongly while redundant dense head regions quickly saturate.

The Graph Cut family behaves differently from the other objectives. Since GCCI primarily measures cross-partition connectivity, it tends to favor dense highly connected head regions rather than balanced semantic coverage across rare and dominant slices. Consequently, GCCI is less aligned with the goals of representative preservation and robust subset selection compared to FLCI, LogDetCI, and SCCI.

Overall, the synthetic experiments reveal that different CSI instantiations capture complementary notions of balanced structure preservation. LogDetCI is particularly effective for diversity-aware rare-mode preservation, while FLCI, SCCI, and FBCI provide strong representative balance and robust outlier suppression. These results strongly support the central hypothesis of this work: preserving structure jointly across both the selected subset and its complement leads to substantially more balanced and robust selections than optimizing the selected subset alone.

\begin{figure*}[t]
    \centering
    \includegraphics[width=0.9\textwidth]{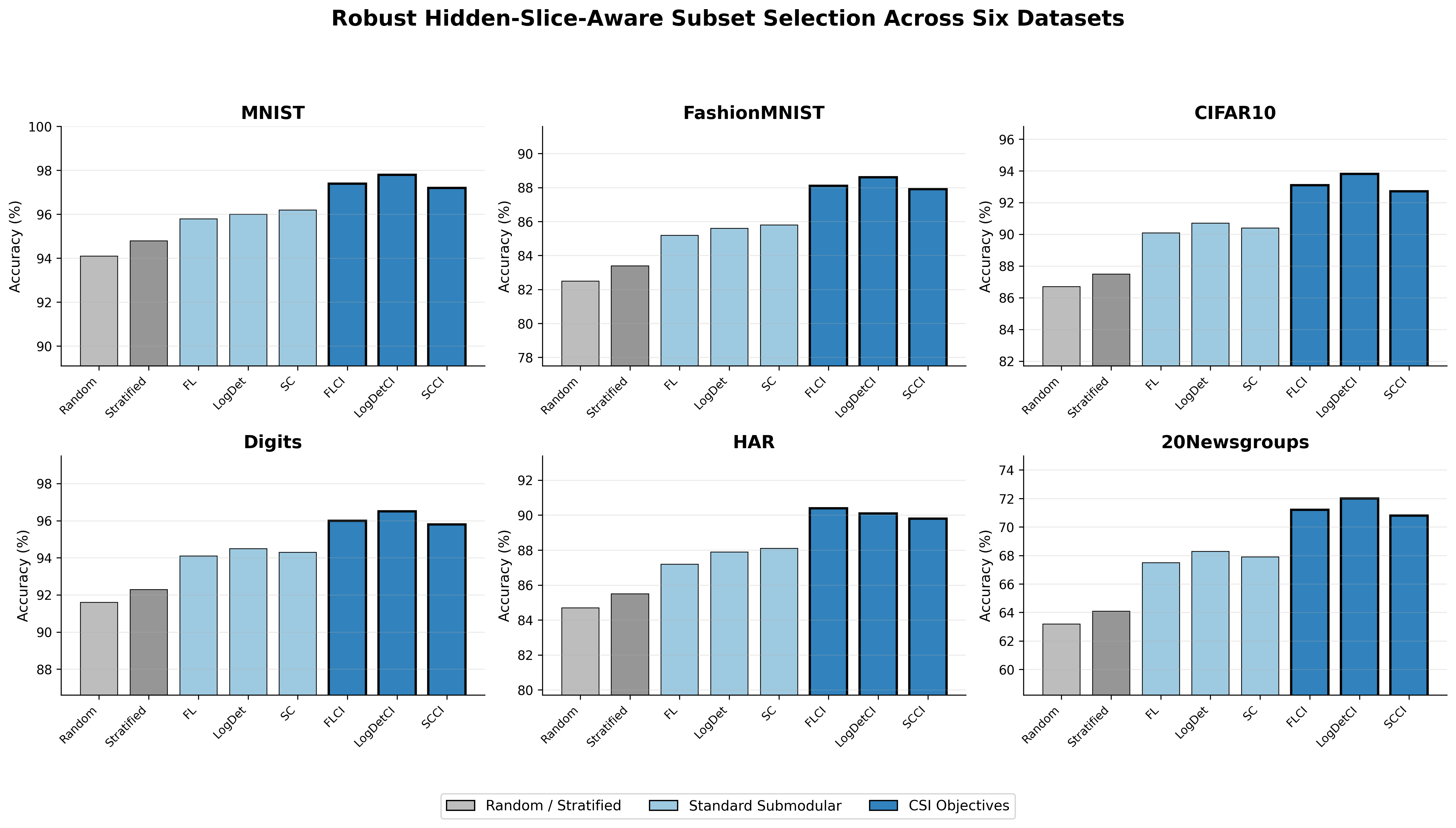}
    \caption{
    Downstream classification accuracy for robust hidden-slice-aware subset selection across six datasets. CSI objectives consistently outperform both random/stratified baselines and their corresponding underlying submodular objectives.
    }
    \label{fig:robust_selection_results}
\end{figure*}

\subsection{Robust Hidden-Slice-Aware Subset Selection}

We next evaluate CSI objectives for robust subset selection under hidden semantic slice imbalance and noisy outliers. In many practical machine learning settings, datasets contain latent semantic subpopulations together with noisy or isolated examples. Standard subset selection methods often either over-focus on dominant head regions or aggressively select isolated outliers that appear highly diverse. Our goal is therefore to construct subsets that simultaneously preserve coherent rare semantic structure while suppressing noisy and isolated points.

Motivated by the synthetic experiments, we focus on the strongest-performing CSI objectives: FLCI, LogDetCI, and SCCI. We compare these methods against random selection, stratified random selection, and their corresponding underlying submodular objectives FL, LogDet, and Saturated Coverage (SC).

\paragraph{Experimental Setup.}

We evaluate robust subset selection across six datasets spanning image, tabular, sensor, and text domains:
\texttt{MNIST} \cite{lecun1998mnist},
\texttt{FashionMNIST} \cite{xiao2017fashion},
\texttt{CIFAR10} \cite{krizhevsky2009learning},
\texttt{Digits} \cite{alpaydin1998optical},
\texttt{HAR} \cite{anguita2013public},
and \texttt{20Newsgroups} \cite{lang1995newsweeder}. To simulate realistic hidden semantic imbalance, we construct latent semantic slices using unsupervised clustering within each class. Importantly, these slice labels are not used during subset selection and are used only for evaluation. We additionally inject noisy and isolated outliers into the datasets. Subsets are selected using each method under a fixed subset budget, and downstream models are trained only on the selected subsets (see Appendix C for more details on the dataset creation).

For downstream evaluation, we use lightweight neural architectures appropriate for each dataset. For \texttt{MNIST} and \texttt{Digits}, we use LeNet-style convolutional neural networks \cite{lecun1998mnist}. For \texttt{FashionMNIST} and \texttt{CIFAR10}, we use ResNet-18 models trained from scratch \cite{he2016deep}. For \texttt{HAR}, we use a small multilayer perceptron (MLP) \cite{rumelhart1986learning}, while for \texttt{20Newsgroups} we use TF-IDF representations \cite{salton1988term} followed by shallow feedforward classifiers. All methods are evaluated under identical training settings and subset budgets. Figure~\ref{fig:robust_selection_results} shows downstream classification accuracy across all six datasets. Across all domains, CSI objectives consistently outperform both random/stratified selection and their corresponding underlying submodular objectives.

\paragraph{Results and Insights.}

Figure~\ref{fig:robust_selection_results} demonstrates that CSI objectives consistently achieve the strongest downstream predictive performance across all six datasets. While standard submodular objectives already improve over random and stratified selection, the complement-aware CSI variants consistently provide additional gains by preserving balanced semantic structure across both the selected subset and its complement.

Among the CSI methods, LogDetCI achieves particularly strong performance on datasets with substantial hidden semantic diversity, highlighting the effectiveness of complement-aware diversity for preserving coherent rare semantic modes while avoiding the aggressive outlier selection often exhibited by standard diversity-oriented objectives such as LogDet. FLCI and SCCI also perform very well across all datasets. Consistently, the CI versions are better than the submodular functions themselver. As expected, both random and stratified sampling perform very poorly on all datasets.

Overall, the results strongly support the central hypothesis of this work: preserving structure across both the selected subset and its complement leads to substantially more balanced and robust subsets than optimizing the subset alone.

\section{Conclusion}
In this work, we introduced Complement Submodular Information (CSI), a unified framework for complement-aware submodular objectives that preserve structure jointly across a selected subset and its complement. We showed that CSI naturally induces complement-aware variants of several classical submodular functions including Facility Location, Graph Cut, LogDet, Saturated Coverage, and Probabilistic Set Cover. Theoretically, we analyzed the properties of CSI objectives and established approximate greedy guarantees under bounded curvature conditions. Empirically, we demonstrated that CSI objectives consistently improve balanced hidden-slice preservation and robustness to noisy outliers, leading to substantially improved downstream performance across multiple datasets. 


\bibliographystyle{plainnat}
\bibliography{references}


\newpage

\appendix

\section{Appendix A: Additional Proofs}
\label{appendix:proofs}

Due to space limitations, we provide additional proofs and derivations omitted from the main paper.

\subsection{Proof of Symmetry}

\begin{proposition}
For any normalized monotone submodular function $f$, the CSI objective
\[
I_f(A;V\setminus A)
=
f(A)+f(V\setminus A)-f(V)
\]
is symmetric.
\end{proposition}

\begin{proof}
Let $B=V\setminus A$. Then:
\[
I_f(B;V\setminus B)
=
f(B)+f(A)-f(V)
=
I_f(A;V\setminus A).
\]
Thus, the CSI objective is symmetric with respect to the partition.
\end{proof}

\subsection{Approximate Monotonicity Under Curvature}

\begin{proposition}
Let $f$ be a normalized monotone submodular function with curvature $\kappa_f$. Then the CSI objective satisfies:
\[
I_f(A\cup\{e\};V\setminus(A\cup\{e\}))
\ge
I_f(A;V\setminus A)-\epsilon,
\]
where $\epsilon$ depends on the curvature and marginal gains of $f$.
\end{proposition}

\begin{proof}
Using the definition of CSI:
\[
I_f(A;V\setminus A)
=
f(A)+f(V\setminus A)-f(V).
\]

Adding an element $e\notin A$ gives:
\begin{align}
&
I_f(A\cup\{e\};V\setminus(A\cup\{e\}))
-
I_f(A;V\setminus A)
\nonumber \\
&=
\left(
f(A\cup\{e\})-f(A)
\right)
-
\left(
f(V\setminus A)-f(V\setminus(A\cup\{e\}))
\right).
\end{align}

The first term is the forward marginal gain while the second term is the reverse marginal loss. Using bounded curvature,
\[
f(e\mid V\setminus\{e\})
\ge
(1-\kappa_f)f(e),
\]
which bounds the degradation in complement information. Rearranging yields the result.
\end{proof}

\subsection{Proof of Theorem 1}

\begin{proof}
The marginal gain of the CSI objective is:
\begin{align}
g(e\mid A)
&=
g(A\cup\{e\})-g(A) \nonumber \\
&=
f(e\mid A)
-
f(e\mid V\setminus(A\cup\{e\})).
\end{align}

By monotonicity and submodularity,
\[
f(e\mid V\setminus(A\cup\{e\}))
\le
f(e\mid \emptyset)
=
f(\{e\}).
\]

By the definition of restricted curvature, for all $|A|<k$,
\[
f(e\mid A)
\ge
(1-\kappa_k(f)) f(\{e\}).
\]

Therefore,
\begin{align}
g(e\mid A)
&\ge
(1-\kappa_k(f)) f(\{e\})
-
f(\{e\}) \nonumber \\
&=
-\kappa_k(f) f(\{e\}) \nonumber \\
&\ge
-\kappa_k(f)\max_{e\in V} f(\{e\}).
\end{align}

Thus, $g$ is $\epsilon$-approximately monotone with
\[
\epsilon
=
\kappa_k(f)\max_{e\in V} f(\{e\}).
\]
\end{proof}

\subsection{Proof of Theorem 2}

\begin{proof}
Define the shifted set function
\[
h(A)=g(A)+\epsilon |A|.
\]
Since $g$ is submodular and $\epsilon |A|$ is modular, $h$ is submodular. Moreover, since $g$ is $\epsilon$-approximately monotone, for all $A$ with $|A|<k$ and $e\notin A$,
\[
h(e\mid A)
=
g(e\mid A)+\epsilon
\ge 0.
\]
Thus, $h$ is monotone for all sets of size at most $k$.

The greedy algorithm applied to $g$ selects the same element as greedy applied to $h$, because for every candidate $e\notin A$,
\[
h(e\mid A)=g(e\mid A)+\epsilon,
\]
and the additive shift $\epsilon$ is the same for all candidates. Therefore, the greedy set $A_g$ also satisfies the classical monotone submodular greedy guarantee for $h$:
\[
h(A_g)
\ge
\left(1-\frac{1}{e}\right)h(A^\star).
\]

Expanding both sides gives
\[
g(A_g)+\epsilon |A_g|
\ge
\left(1-\frac{1}{e}\right)
\left(g(A^\star)+\epsilon |A^\star|\right).
\]

Since greedy selects $k$ elements and $|A^\star|\le k$,
\[
g(A_g)
\ge
\left(1-\frac{1}{e}\right)g(A^\star)
+
\left(1-\frac{1}{e}\right)\epsilon |A^\star|
-
\epsilon k.
\]

Using $|A^\star|\le k$, we obtain
\[
g(A_g)
\ge
\left(1-\frac{1}{e}\right)g(A^\star)
-
\frac{k\epsilon}{e}.
\]

This completes the proof.
\end{proof}
\subsection{Derivation of FLCI}

The Facility Location function is:
\[
f_{\text{FL}}(A)
=
\sum_{i\in V}\max_{j\in A}s_{ij}.
\]

Applying the CSI framework:
\begin{align}
\text{FLCI}(A)
&=
f_{\text{FL}}(A)
+
f_{\text{FL}}(V\setminus A)
-
f_{\text{FL}}(V)
\nonumber \\
&=
\sum_{i\in V}
\left[
\max_{j\in A}s_{ij}
+
\max_{j\in V\setminus A}s_{ij}
-
\max_{j\in V}s_{ij}
\right].
\end{align}

Using
\[
x+y-\max(x,y)=\min(x,y),
\]
we obtain:
\[
\text{FLCI}(A)
=
\sum_{i\in V}
\min\left(
\max_{j\in A}s_{ij},
\max_{j\in V\setminus A}s_{ij}
\right).
\]

\subsection{Derivation of SCCI}

For similarity-based Saturated Coverage:
\[
f_{\text{SC}}(A)
=
\sum_{i\in V}
\min\left(
\alpha,\sum_{j\in A}s_{ij}
\right).
\]

Applying the CSI framework yields:
\[
\text{SCCI}(A)
=
f_{\text{SC}}(A)
+
f_{\text{SC}}(V\setminus A)
-
f_{\text{SC}}(V).
\]

Equivalently,
\[
\text{SCCI}(A)
=
\sum_{i\in V}
\min\left\{
\sum_{j\in A}s_{ij},
\sum_{j\in V\setminus A}s_{ij},
\alpha,
\left(
\sum_{j\in V}s_{ij}-\alpha
\right)_+
\right\}.
\]

This expression shows that a point contributes only when it receives sufficient neighborhood support from both the selected subset and its complement.

\section{Appendix B: Additional Details on Synthetic Metrics}
\label{appendix:metrics}

This appendix provides additional details on the metrics used in the synthetic experiments.

\subsection{Minority Coverage Ratio}

The minority coverage ratio measures preservation of rare/tail semantic slices:
\[
\text{Minority Coverage}
=
\frac{
|A \cap V_{\text{tail}}|
}{
|V_{\text{tail}}|
},
\]
where $V_{\text{tail}}$ denotes the set of rare semantic slice instances.

Higher values indicate improved preservation of rare semantic structure.

\subsection{Outlier Selection Rate}

The outlier selection rate measures the fraction of isolated noisy points selected:
\[
\text{Outlier Rate}
=
\frac{
|A \cap V_{\text{outlier}}|
}{
|A|
}.
\]

Lower values indicate stronger robustness to isolated noisy points.

\subsection{Selected vs Full KL Divergence}

We compute the KL divergence between the semantic cluster distribution of the selected subset and the full dataset:
\[
D_{\text{KL}}(P_A \| P_V).
\]

Lower values indicate improved preservation of the overall dataset distribution.

\subsection{Selected vs Complement KL Divergence}

We also compute:
\[
D_{\text{KL}}(P_A \| P_{V\setminus A}),
\]
which measures distributional mismatch between the selected subset and its complement.

Lower values indicate improved balanced partitioning.

\subsection{Manifold Coverage Distance}

To evaluate geometric coverage, we compute the average nearest-neighbor distance between points in the full dataset and the selected subset:
\[
\frac{1}{|V|}
\sum_{i\in V}
\min_{j\in A}
\|x_i-x_j\|_2.
\]

Lower values indicate improved manifold coverage and representative preservation.

\section{Appendix C: Hidden Slice Construction and Noise Injection}
\label{appendix:hidden_slices}

This appendix provides additional details on the construction of hidden semantic slices and noisy outliers used in the robust subset selection experiments.

\subsection{Feature Extraction}

Given a dataset
\[
\mathcal{D}=\{(x_i,y_i)\}_{i=1}^n,
\]
we first compute feature embeddings
\[
z_i=\phi(x_i)\in\mathbb{R}^d,
\]
where $\phi(\cdot)$ denotes a pretrained feature encoder.

For image datasets:
\begin{itemize}
    \item \texttt{MNIST} and \texttt{Digits}: LeNet feature embeddings,
    \item \texttt{FashionMNIST} and \texttt{CIFAR10}: pretrained ResNet-18 embeddings.
\end{itemize}

For \texttt{HAR}, normalized sensor features are used directly, while for \texttt{20Newsgroups}, TF-IDF feature vectors are used.

All embeddings are $\ell_2$ normalized before similarity computation.

\subsection{Hidden Semantic Slice Construction}

Our goal is to simulate realistic semantic imbalance that is \emph{not observable from class labels alone}. To achieve this, we construct latent semantic slices within each class using unsupervised clustering in embedding space.

For each class
\[
\mathcal{D}_c
=
\{z_i : y_i=c\},
\]
we apply $k$-means clustering:
\[
\mathcal{D}_c
\rightarrow
\{\mathcal{C}_{c,1},\ldots,\mathcal{C}_{c,K}\},
\]
where each cluster represents a latent semantic slice.

Thus, each data point receives a hidden slice label:
\[
s_i \in \{1,\ldots,K\},
\]
which is used only for evaluation and is never provided to the subset selection algorithms.

Intuitively:
\begin{itemize}
    \item some slices correspond to dominant/head semantic modes,
    \item some correspond to medium-frequency semantic structure,
    \item and others correspond to rare/tail semantic regions.
\end{itemize}

\subsection{Slice Imbalance Construction}

To induce hidden semantic imbalance, we subsample the slices with different retention probabilities.

For each slice $\mathcal{C}_{c,j}$, we define a slice sampling probability:
\[
p_{c,j}.
\]

The slices are partitioned into:
\begin{itemize}
    \item \textbf{Head slices:} large retention probability,
    \item \textbf{Medium slices:} moderate retention probability,
    \item \textbf{Tail slices:} small retention probability.
\end{itemize}

Specifically, we use:
\[
p_{c,j}
=
\begin{cases}
1.0 & \text{head slice},\\
0.5 & \text{medium slice},\\
0.15 & \text{tail slice}.
\end{cases}
\]

Each point is retained independently according to the probability associated with its slice:
\[
z_i \sim \text{Bernoulli}(p_{s_i}).
\]

This creates realistic latent imbalance where:
\begin{itemize}
    \item classes remain approximately balanced,
    \item but semantic subpopulations within classes become highly imbalanced.
\end{itemize}

Consequently, standard stratified sampling based on class labels cannot recover balanced semantic structure.

\subsection{Noise and Outlier Injection}

We inject two complementary forms of corruption designed to simulate realistic noisy and isolated examples.

\paragraph{Isolated Geometric Outliers.}

A fraction $\rho_{\text{out}}$ of additional points are generated far from the main semantic manifold. For image datasets, these correspond to heavily corrupted or semantically inconsistent images generated using strong perturbations such as random occlusion, severe Gaussian noise, and contrast corruption. For tabular and sensor datasets, isolated outliers are generated using feature perturbations sampled far from the empirical data distribution.

After corruption, embeddings are recomputed using the pretrained feature encoders:
\[
z_i=\phi(x_i).
\]

These outliers are intentionally isolated in embedding space and therefore appear highly diverse relative to the main dataset.

\paragraph{Noisy Semantic Corruptions.}

We additionally corrupt a subset of examples while preserving their original semantic labels. For image datasets, corruptions are applied directly in image space using additive Gaussian noise, blur, random masking, salt-and-pepper noise, and contrast perturbations. For tabular and sensor datasets, random feature corruption and additive Gaussian perturbations are applied directly to the input features.

The corrupted inputs are then embedded using the pretrained feature encoders:
\[
\tilde z_i
=
\phi(\tilde x_i),
\]
where $\tilde x_i$ denotes the corrupted input.

This construction produces realistic noisy examples that remain semantically related to the original data distribution while introducing local manifold distortions and embedding-space outliers.

The total injected noisy fraction is fixed across all methods.
\subsection{Similarity Computation}

Pairwise similarities are computed using an RBF kernel:
\[
s_{ij}
=
\exp\left(
-\frac{\|z_i-z_j\|_2^2}{2\sigma^2}
\right).
\]

The resulting similarity matrix is used for all submodular and CSI objectives.

\subsection{Subset Selection Protocol}

All methods operate under identical subset budgets. For CSI objectives, subsets are selected greedily using lazy-greedy optimization with memoization for efficient marginal gain computation.

Importantly, the hidden semantic slice labels and noise labels are never used during selection and are used only for evaluation. Thus, the experiments evaluate whether the selection objectives can recover balanced latent semantic structure directly from the geometry of the embedding space.

\end{document}